\documentclass[letterpaper, 10 pt, conference]{ieeeconf}  % Comment this line out if you need a4paper

\usepackage[pdftex]{graphicx}        % standard LaTeX graphics tool
\usepackage{mathptmx}
\usepackage{makeidx}
\usepackage{multicol}
\usepackage{footmisc}
\usepackage{textcomp}
\usepackage{algorithm}
\usepackage{algpseudocode}
\usepackage{mathrsfs}
\usepackage{paralist}
\usepackage[font=small,format=plain,labelfont=bf]{caption}
\usepackage{amssymb,amsmath}
\usepackage{theorem}
\usepackage{footnote}
\usepackage{subfig}
\usepackage{multirow}
\usepackage{array}
\usepackage{float}
\usepackage[usenames,dvipsnames]{color}
\usepackage{bm}
\usepackage{mdwlist}
\usepackage{cite}

\newcommand{\REMOVED}[1]{ }

\newcommand{\todo}[1]{
  \textcolor{red}{\footnotesize \textsf{#1}}
}
\renewcommand{\todo}[1]{} %% remove comments to hide TODOs

\newcommand{\topic}[1]{
  \textcolor{Emerald}{\footnotesize \textsf{#1}}
}
\renewcommand{\topic}[1]{} %% remove comments to hide TODOs

\newcommand{\figDim}{1.0}
\newcommand{\figDimT}{1.0}

\newtheorem{theorem}{Theorem}[section]
\newtheorem{lemma}[theorem]{Lemma}

\newtheorem{corollary}[theorem]{Corollary}

\IEEEoverridecommandlockouts                              % This command is only needed if
\begin{document}
\title{\LARGE\bf Local Policies for Efficiently Patrolling\\ a Triangulated Region by a Robot Swarm}
%\title{\LARGE\bf Efficiently Patrolling a Region with a Robot Swarm}

\author{Daniela Maftuleac$^{1}$, Seoung Kyou Lee$^{2}$, S\'andor~P.~Fekete$^{3}$, Aditya Kumar Akash$^{4}$,\\ Alejandro L{\'o}pez-Ortiz$^{1}$, and James McLurkin$^{2}$
%\thanks{*This work was not supported by any organization}% <-this % stops a space
\thanks{$^{1}$Daniela Maftuleac and Alejandro L{\'o}pez-Ortiz at Cheriton School of Computing,
         University of Waterloo, Waterloo, ONT, Canada.
        {\tt\small dmaftule,alopez-o@cs.uwaterloo.ca}}%
\thanks{$^{2}$Seoung Kyou Lee and James McLurkin at Computer Science Department,
        Rice University, Houston, TX, USA.
        {\tt\small sl28,jmclurkin@rice.edu}}%
\thanks{$^{3}$S{\'a}ndor P. Fekete at Department of Computer Science, TU Braunschweig, Braunschweig, Germany.
        {\tt\small s.fekete@tu-bs.de}}%
\thanks{$^{4}$Aditya Kumar Akash at Department of Computer Science and Engineering, IIT Bombay, Mumbai, India.
{\tt\small adityakumarakash@gmail.com}}
}

\maketitle
%%%%%%%%%%%%%%%%%%%%%%%%%%%%%%%%%%%%%%%%%%%%%%%%%%%%%%%%%%%%%%%%%%%%%%%%%%%%%%%%
\begin{abstract}
We present and analyze methods for patrolling an environment with a distributed swarm of robots.  Our approach uses a \emph{physical data structure}~--~a distributed triangulation of the workspace.  A large number of stationary ``mapping'' robots cover and triangulate the environment and a smaller number of mobile ``patrolling'' robots move amongst them.  The focus of this work is to develop, analyze, implement and compare local patrolling policies. We desire strategies that achieve full coverage, but also produce good coverage frequency and visitation times.  Policies that provide theoretical guarantees for these quantities have received some attention, but gaps have remained.

We present:
\textbf{1)} A summary of how to achieve coverage by building a triangulation of the workspace, and the ensuing properties.
\textbf{2)} A description of simple local policies (LRV, for {\em Least Recently Visited} and LFV, for {\em Least Frequently Visited}) for achieving coverage by the patrolling robots.
\textbf{3)} New analytical arguments why different versions of LRV may require worst-case exponential time between visits of triangles.
\textbf{4)} Analytical evidence that a local implementation of LFV on the {\em edges} of the dual graph is possible in
our scenario, and immensely better in the worst case.
\textbf{5)} Experimental and simulation validation for the practical usefulness of these policies, showing that even a small number of weak
robots with weak local information can greatly outperform a single, powerful robots with full information and computational capabilities.
\end{abstract}

%%%%%%%%%%%%%%%%%%%%%%%%%%%%%%%%%%%%%%%%%%%%%%%%%%%%%%%%%%%%%%%%%%%%%%%%%%%%%%%%

\section{Introduction and Related Work}
\label{sec:Introduction}

%The story:
%\begin{itemize}
	%\item Previous Work LRV on edges.  It was good.
	%\item Problem: Exponential lower bound on LRV - this is bad. (this is new result)
	%\item Obvious solution LFV-v should be the answer, but indication that it is bad - open problem on theory side
	%\item Solution: LFV-e: quadratic upper bound number of edges x diameter, diameter might be sqrt n
	%\item LFV-e with accounting on visitation rates should be good, yes?
%\end{itemize}

\begin{figure}[t]
\centering
\includegraphics[width=.8\linewidth]{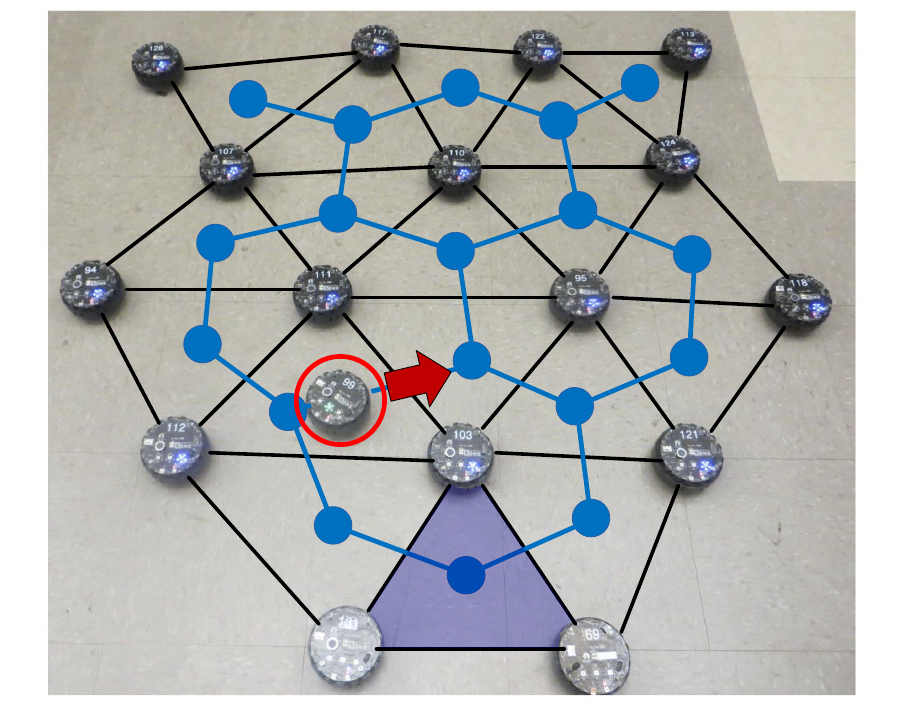}
\caption{
\label{fig:NavPatrol}
Triangulated network patrolling experiment with one patrolling robot (red circle). Blue
lines indicate the dual graph of the triangulation.  The blue triangle is the starting point for the patrol.  The patrol policy computes the next adjacent triangle a patrolling robot should visit.  It's output is shown with the red arrow near the patrolling robot.
}
\end{figure}

\topic{Introduction}
Large populations of robots are ideal for tasks where the robot must cover a
large geographic area, such as search-and-rescue, exploration, mapping and
surveillance.  The robots can maintain coverage of the environment after the
dispersion is complete.  The size of the environment that can be covered is
proportional to the population size, but large populations require that
individual robot be quite simple, without expensive sensors or computation. In
this paper, we focus on using a heterogeneous group of robots; with many small
``mapping'' robots that map the environment and build a communication network,
and a smaller (but still numerous) number of more capable ``patrolling'' robots
with the capability to respond to events.  After deployment of the mapping
robots, controlling the more powerful patrolling robots amounts to a
\emph{coverage control} problem: How should the navigating robots move in order
to ensure small worst-case latency in patrolling all areas of the surveyed
environment? Given the
distributed nature of a swarm, this requires simple local strategies that do
not involve complicated protocols or computations for coordinating the motion
of the mobile components,
while still achieving complete coverage, with small latency to surveyed locations.

\topic{The Problem}
%There are many approaches to coverage control ~\cite{cortes2004coverage, bhattacharya2013distributed, breitenmoser2010voronoi, caicedo2008coverage}.
%\todo{first sentence removed (see source) We need to explain/contrast/compare these refs or remove them.  I vote for removal.(james)}
In previous work~\cite{bfk+-tueur-13, flm+-prssr-14}, we showed how complete
coverage of an unknown region can be achieved by performing a structured
exploration by a multi-robot system with bearing-only low-resolution sensors.
The result is a triangulation of the workspace that can be exploited for
further tasks.  As described in \cite{lee2014} this supports a straightforward
approach to patrolling: Each triangle can be considered a vertex in a dual
graph, with adjacent triangles connected by dual edges. Thus, any route in the
workspace that visits a sequence of triangles can be traced by a path in the
dual graph.  This can significantly reduce the computation on the patrolling
robots~--~a simple policy that considers the current triangle and adjacent
triangles will suffice.  Given well-shaped triangles (which can be achieved by
the robot platform described in this paper), a policy can produce a patrol with
provable properties.  Fig.~\ref{fig:NavPatrol} shows an example experiment.

What local policies should be used for patrolling the triangulated region? A
natural choice for this task is {\em Least Recently Visited} (LRV), in which
each triangle keeps track of the time elapsed since its last visit from a
patrolling robot.  The patrolling robot policy directs it to move to the
adjacent triangle with the smallest such latency. This amounts to tracking the
visit times of triangles, which are the vertices in the dual graph, the blue
circles in Fig.~\ref{fig:NavPatrol}. We show that this yields a policy that
achieves complete coverage, and no obvious problems in practical
experiments\cite{lee2014}.

\topic{Contributions}
In this paper, we investigate the theoretical properties of patrolling
policies. We present new results that show using LRV on dual {\em vertices}
(i.e., LRV-v) can perform quite badly, by proving that the resulting coverage
time can be exponential in the number $n$ of dual vertices. This is a new
analytical result for our specific class of graphs that arise from planar
triangulations, i.e., planar graphs of maximum degree three. We present
alternate policies that share the simple local information requirements of
LRV-v, while achieving latencies that are small even in the worst case. The
policy {\em Least Frequently Visited} (LFV) that keeps track of the {\em
frequency} of visits is such a candidate. Making this policy provably good
requires a particular twist: instead of tracking visits to triangles, the dual
vertices, we track the use of edges between triangles, the dual edges.  We call
these two new policies  LFV-v and LFV-e respectively.
We augment the {\em physical data structure} of our triangulated network to
maintain these frequency counts, and show that this data can be stored and
retrieved by patrolling robots with fixed-size communication messages.  This
leads to latencies that are well bounded, even in the worst case: for regions
covered by $n$ triangles, the latency of LFV-e is not worse than $O(n\cdot d)$,
where $d$ is the dual diameter of the region, i.e., the largest number of
triangles that may have to be visited in a shortest path. For naturally shaped
regions with bounded aspect ratio, this amounts to an upper bound of
$O(n^{1.5})$.  We present simulation results and hardware experiments with
19 robots that demonstrate the efficacy of our approach.

%This paper presents \todo{Possibly Update from abstract to make it more suitable for intro}:
%\begin{itemize*}
%\item A description of a real-world robot platform (the r-one robots) with limited capabilities that allows coverage, communication and mobility
%\item A summary of how to achieve coverage by building a triangulation of the workspace, and the ensuing properties.
%\item A description of a simple local policies (LRV, for {\em Least Recently Visited}) for achieving coverage by the patrolling robots.
%\item A description of an alternative local policy (LFV, for {\em Least Frequently Visited}).
%\item New analytical arguments why LRV may require exponential time between visits of triangles.
%\item Analytical evidence that a local implementation of LFV on the {\em edges} of the dual graph is possible,
%and immensely better in the worst case.
%\item Experimental validation for the practical usefulness of our results.
%\item Hardware and simulation experiments.
%\end{itemize*}

\subsection{Assumptions}
We focus our attention on approaches applicable to small, low-cost devices with
limited sensors and capabilities.  In this work, we assume that robots do not
have a map of the environment, nor the ability to localize themselves relative
to the environment geometry, \emph{i.e.} SLAM-style mapping is beyond the
capabilities of our platform.  Instead, we assume that the mapping robots can disperse and triangulate the environment; see \cite{bfk+-tueur-13} for an illustrative video, and \cite{lee2014} for a technical paper. We exclude solutions that use centralized
control, as the communication and processing constraints do not allow these
approaches to scale to large populations.  We also do not assume that GPS
localization or external communication infrastructure is available, which are
limitations present in an unknown indoor environment.
We assume that the communication range is much smaller than the size of the environment, so a
multi-hop network is required for communication.  Finally, we assume that the
devices know the geometry of their local communications network.  This
\emph{local network geometry} provides each robot with relative pose
information about its neighbors.

%\todo{Text removed (see source) do we need it back?}
%The main objective is to perform patrolling by the mobile robots, i.e., a
%protocol that makes sure the robots circulate within the region, visiting
%portions at regular intervals, and are able to respond to events when they happen.
%Clearly, these objectives can be carried out in different ways:
%(1) We can aim for strictly local policies that ensure {\em all} dual vertices
%get visited infinitely often.
%(2) We can aim for strictly local policies that ensure the ``important'' vertices
%get visited frequently, so that the response time to an event can be kept small.

\subsection{Related Work}
Our results rely on the computational power of many small robots distributed
throughout the environment, which support many basic algoritms.  The patrolling
robots use the mapping robot's network for navigation, there are many
references, we note that Batalin's approach is similar to our
own~\cite{batalin_usingsensor_2004}.  Our network is composed of triangles,
which provide useful geometric properties. Approaches like those of Spears et
al.~\cite{w._m._spears_distributed_2004} build a triangulated configuration
using potential fields, but the network does not have a physical data
structure, so the robots never recognize that they form triangles.  Our
approach allows us to use triangles as computational elements, which support
practical distributed computations\cite{lee2014}.  Geraerts~\cite{geraerts2010planning}
or Kallmann~\cite{kallmann2005path}, use a
triangulated environment for path planning, but require global information and
localization.  Our approach is fully distributed, using only local information
and communications.

Optimizing the refresh frequency when patrolling a graph even by a single,
powerful robot with full information amounts to finding a shortest roundtrip
that visits all vertices -- the well-known {\em Traveling Salesman Problem}
(TSP), which is known to be NP-hard, even for full information and central
control; see the book~\cite{applegate2006tsp} for a comprehensive study of solution methods,
and the book~\cite{cook2012tsp} for a recent overview of history and aspects of optimization.
Simple approximation algorithms for the TSP that require a limited amount of
computation do exist, for example, based on building a minimum spanning
tree~\cite{eak-mrapf-09}; however, this still requires global information or communication, even
with only local connections.  Moreover, the problem for multiple robots amounts
to the {\em Vehicle Routing Problem} (VRP), for which the additional task of finding
a well-balanced distribution of vertices visited by the individual robots
impedes performance guarantees for simple heuristics, such as doubling a
spanning tree.  This
contributes to making the VRP very challenging in practical contexts. See the book by Toth and
Vigo~\cite{toth2001vehicle} for a comprehensive overview.

Motivated by using a swarm of weak robots, rather than powerful centralized methods,
we favor approaches that require only local information, which can be
maintained by the dual vertices.  The most basic policy is to use a random
walk for each robot. As discussed in detail by Cooper et
al.~\cite{cik+-drwug-11}, this has some obvious disadvantages in the worst
case, and may even be of limited quality in the average case.  Better approaches consider available information, such as the time elapsed since the
last visit by a robot. This has been considered in the context of token passing
in decentralized ad-hoc networks. As Malpani et al.~\cite{mcv+-dtcma-05}
showed, the policy {\em Least Recently Visited} (LRV) ensures that finite
refresh times can be guaranteed. However, one of the theoretical results by
Cooper et al.~\cite{cik+-drwug-11}
was to show that there are examples for which this policy may
result in refresh times that are exponential in the size of the graph. More on this will be discussed in Section~\ref{sec:policy},
where we discuss LRV in the context of mobile robots.

%In the context of patrolling, it may be more relevant to circulate among a set of central locations from which every other
%portion of the region can be reached relatively quickly when the need arises, rather than visiting all nodes at larger interval.
%We will propose a policy that makes use of only local information, and has some resemblance to the well-known {\em Page Rank};
%see \cite{Page99thepagerank}, and \cite{pf-scmcn-08} for a study of this and other centrality measures.

\section{Model and Preliminaries}
\label{sec:ModelAndAssumptions}

%\paragraph{Robot Model}
We have a system of $p$ mapping robots and $r$ patrolling robots, where
$p \gg r$.  The communication network is an undirected graph $G=(V,E)$. Each
robot is modeled as a vertex, $u \in V$, where $V$ is the set of all robots and
$E$ is the set of all robot-to-robot communication links.
%(Note that we discuss a subgraph formed by only the triangulation robots in the next subsection.)
The neighbors of each vertex $u$ are the set of robots within line-of-sight communication range
$r_{max}$ of robot $u$, denoted $N(u)=\{v \in V\ \mid \{u,v\} \in E\}$. We
assume all network edges are also navigable paths.  Robot $u$ sits at the
origin of its local coordinate system, with the $\hat{x}$-axis aligned with its
current heading.   Robot $u$ can measure the relative pose of its neighbors in its reference frame.
We model algorithm execution as a series of synchronous \emph{rounds}.  This simplifies analysis and is straightforward to implement in a physical system~\cite{mclurkin_analysis_2008}.
%At the end of each round, every robot $u$ broadcasts a message to all of its neighbors. The robots randomly offset their initial transmission to minimize collisions.

%is doesn't mean that all robots broadcast message at the same time. Rather, each robot periodically  claims its own time span, called round, among the whole time peroid, called bandwidth $t_B \leq n t_{ro}$, to run algorithm and broadcast a message to avoid collision.
%During the duration of each round, robot $u$ receives a message from each neighbor $v \in N(u)$. Each message contains a set of public variables, including the sending robot's unique ID number $u.id$. The remaining variables will be defined later, but we note that the number of bits needed for each variable is bounded by $\log_2 n$, i.e. the number of bits required to identify each robot. This produces a total message of constant size.

%\paragraph{Environment Triangulation}
\label{subsec:DualGraphNavigation}

%Figure~\ref{fig:EnvironmentTriangulation} shows two snapshots
The mapping robots use our MATP
triangulation algorithm~\cite{fkk+-etrsr-11, bfk+-tueur-13, lee2014}, to explore and triangulate the environment.  The exploration proceeds in a breadth-first
fashion, leaving a triangulated network in its wake.  Figure~\ref{fig:policy} shows an example triangulation of mapping robots.  This {\em primal graph} $G_P=(V_P,E_P)$
consists of the $p$ \emph{triangulation} robots, $V_P$, constructing the network, and the edges, $E_P$, forming the triangles. Pairs of triangulation robots with $\{v_i, v_j\}\in E_P$
are able to communicate with each other, while any mobile robot $r$ inside a triangle formed by robots $v_i$, $v_j$, $v_k$ can communicate with all three of them.
In turn, the set $V_D$ of triangles forms the vertices of a {\em dual graph} $G_D=(V_D,E_D)$, in which two vertices $\Delta_i,\Delta_j\in V_D$ are connected by an edge
if the triangles represented by $\Delta_i$ and $\Delta_j$ are adjacent. The triangulation algorithm guarantees that a dual edge $\{\Delta_i,\Delta_j\}\in E_D$ corresponds to the primal edge $\{v_i, v_j\}\in E_P$, where $v_i$ is the {\em owner} of triangle $\Delta_i$.  This allows us to model computation on $\Delta_i$ and communication between $\Delta_i$ and $\Delta_j$ while the actual computation and communication is on $v_i$ and $\{v_i, v_j\}\in E_P$.  See \cite{lee2014} for details.

\section{Local Patrolling Policies}
\label{sec:policy}

\subsection{Maintaining the Dual Graph}
In the following section we focus on the dual graph, and consider the triangles as computational elements.
%However, actual processing and communication can only be carried out on the stationary robots. We can implement this by processing and
%storing all information regarding a triangle on one of their stationary robots, the designated {\em owner} of the triangle.
%Ownership is assigned during the initial BFS process of placing the robots: a robot becomes owner of a triangle
%when he completes it. Because using well-shaped triangles (with guaranteed lower bound on the minimum angle)
%implies a small bounded number of triangles incident to each stationary robot, we can be assured that this only requires
%storing and processing a moderate amount of information by each stationary robot.
However, computation in triangles actually occurs on the primal vertices (the robots).  A crucial property is the following, established in~\cite{lee2014}.

\begin{theorem}
\label{ConnectedOwners}
The owners of two adjacent triangles must also be connected.
\end{theorem}

Thus, a mobile robot in a triangle $\Delta_i$ can do status checks on neighboring triangles (i.e., neighbors
of $\Delta_i$ in the dual graph) by asking the owner
of $\Delta_i$ to query its neighbors in the primal graph.

In our new algorithms, it becomes important to consider local patrolling policies that are based
on keeping track of dual {\em edges} instead of vertices. This makes it necessary to store and retrieve
information on these dual edges. We can implement this by making the owner of a triangle also the owner
of the dual edges to all adjacent triangles.  It is important to observe that these owners of a dual edge
are {\em not} the vertices of the corresponding primal edge. Implementing this on the stationary robots forming the primal
graph can be based on the following result.

\begin{theorem}
\label{ConnectedEdges}
Any dual edge has two owners that are connected in the primal graph.
\end{theorem}

For a proof, observe that any dual edge involves two triangles, whose owners are adjacent by Theorem~\ref{ConnectedOwners}.
Moreover, it also follows that any primal vertex can only be the owner of a small number of dual edges, which bounds memory and communication usage. This allows us to model information stored in a distributed fashion on triangles and edges, and reason
about communication via the dual graph.

%Establishing performance guarantees, both theoretical and practical, relies on the quality of the underlying triangulation.  Let $r_{\text{max}}$ be the maximum length of a triangulation edge. We also consider a lower bound of $r_{\text{min}}$ on the length of the shortest edge in the triangulation; in particular,
%we assume that the local construction ensures that any non-boundary edge is long enough to let a robot pass between the two robots marking the vertices
%of the edge, so $r_{\text{min}}\geq 2\delta$, where $\delta$ is the diameter of a robot. Finally, angular measurements of neighbor positions let us guarantee a minimum angle of $\alpha$ in all triangles.  These constraints give rise to the following:
%
%\begin{definition}
%\label{def:alphaFat}
%Let $\mathcal T$ be a triangulation of a planar region $\mathcal R$, with vertex set $V$. $\mathcal T$ is $(\rho,\alpha)${\em -fat}, if
%it satisfies the following properties:
%\begin{itemize*}
%\item The ratio $r_{\text{max}}/r_{\text{min}}$  of longest to shortest edge in $\mathcal T$ is bounded by some positive $\rho$.
%\item All angles in $\mathcal T$ have size at least $\alpha$.
%\end{itemize*}
%\end{definition}
%
%\noindent This definition is used to prove properties of the coverage control based on tessellations.

%\subsection{Navigation Guarantees}
%\label{subsec:Guarantees}
%
%Summarize stretch factors, if necessary (and if there's space.)
%This section describes the policy.
\begin{figure}
\centering
\includegraphics[width=.70\linewidth]{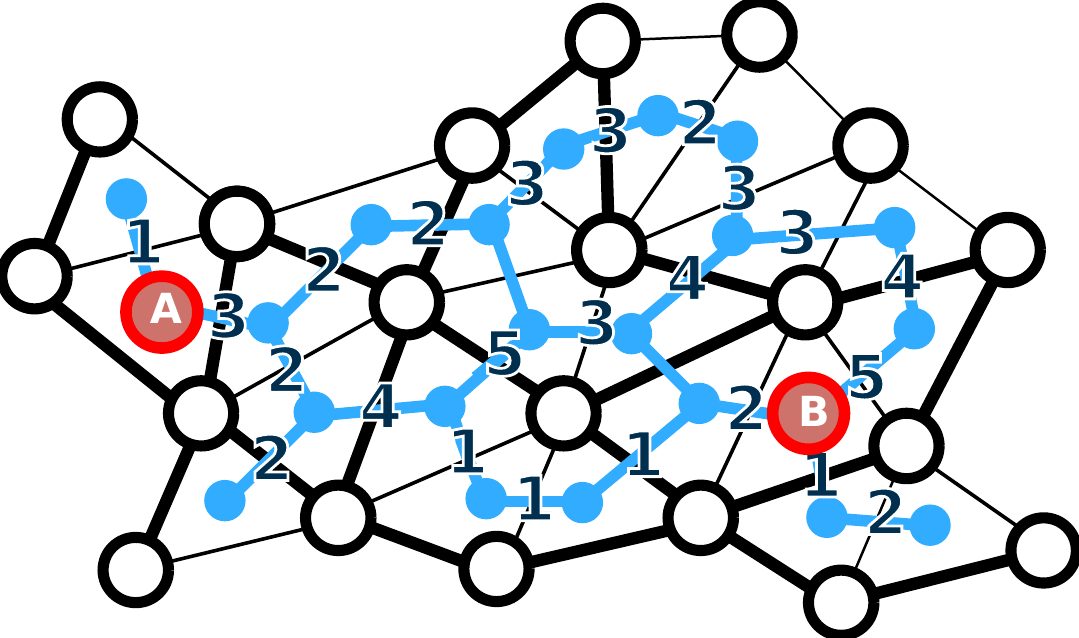}
\caption{
\label{fig:policy}
Notation for triangulation with a dual graph. {\bf Black circles:}triangulation
robots, {\bf red circles:}patrolling robots, {\bf blue lines} dual graph,
{\bf black lines:} primal graph edges that map to dual graph edges, {\bf thick black lines:} subset of dual edges that indicate connections
between owner vertices. The {\bf
black numbers} indicate the current data stored in each edge.  For example, in
LFV-e patrolling (Sec.~\ref{subsec:graphs}), this data would be the frequency
of visits for that edge.
%The data stored in the dual graph is actually stored in one of the stationary
%robots that form the adjacent triangles.
}
\end{figure}

\subsection{Basic Policies}
\label{subsec:graphs}
We consider a number of different local policies that allow patrolling
all triangles of the environment. The objective is to minimize maximum latency, i.e., the longest refresh time ($RT$)
between consecutive visits of the same triangle.

Our previous work described a local patrolling policy that moves each patrolling robot into the adjacent triangle
with the largest refresh time.  We refer to this policy as LRV-v, for {\em least recently
visited vertex}).  This design makes sense, as the objective
is to keep refresh times small. While simple, this policy produces complete coverage~\cite{mcv+-dtcma-05}.

The LRV policy has been studied in various contexts.
Fig.~\ref{fig:expData} demonstrates that it exhibits relatively good behavior in
practice.  The experiment starts with one navigating
robot; we add others as time proceeds.  As we deploy more navigating
robots, the maximum $RT_t({\Delta_i})$ decreases, as expected.

An alternative to considering the {\em time} since the last visit to a triangle
is to keep track of the {\em frequency} of visiting. The rationale behind this is
that an even distribution of visits should make the maximum latency close to the
average. Thus, we obtain the policy LFV-v, for {\em least frequently visited vertex}.

Keeping track of visits to dual vertices (i.e., triangles) is natural,
but not the only possible choice. Instead, we can track visits to dual edges,
giving rise to the policies LRV-e and LFV-e. As it turns out, the crucial outcome of this paper
is that both LRV-v and LFV-e have worst-case exponential latency, so they should be
used only with care. While precise proofs on the worst-case behavior of LFV-v have yet
to be established, we present evidence that it may also be bad. On the positive side,
the worst-case behavior of LFV-e {\em can} be bounded, making it the unique policy
of choice when trying to bound worst-case behavior.

\subsection{Worst-Case behavior of LRV-e and LRV-v}

It is known that the worst-case behavior of LRV-e in arbitrary graphs can be
exponential in the number of nodes in the graph, provided we allow a maximum
degree of at least 4. That is, for every $n$ there exists a graph with $n$
vertices in which the largest refresh time for a node is
$\exp(\Theta(n))$~\cite{cik+-drwug-11}.  Fig.~\ref{fig:reg_graph} depicts one
such graph (with vertices of degree 4), which filters a fixed percentage (1/3 to be precise) of all
left-to-right paths that go past the diamond-like gadgets. If we connect
$\Theta(n)$ such gadgets in series, we will require a total of
$(3/2)^{\Theta(n)}$ paths,
starting from the left for at least one of them to reach the rightmost point in the series.

\begin{figure}
\centering
\includegraphics[width=0.34\textwidth]{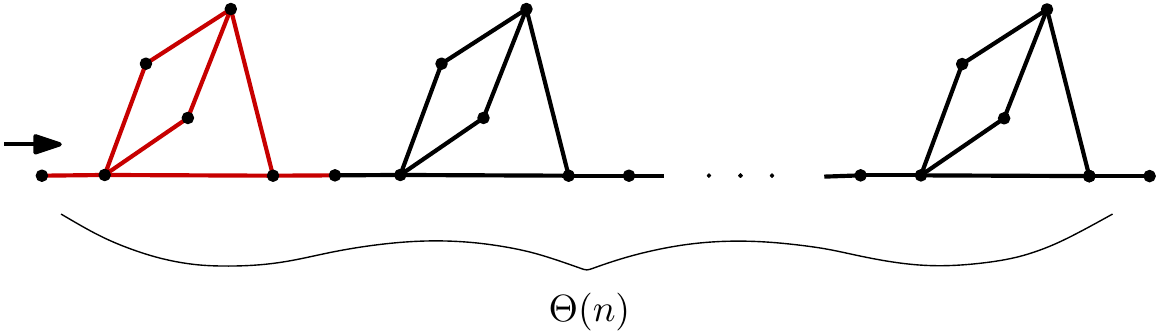}
\caption{Graph with $n$ vertices with a chain of $\Theta(n)$
gadgets. A single gadget is colored in red for illustration purposes. Patrolling takes exponential time in the worst case~\cite{cik+-drwug-11}.} \label{fig:reg_graph}
\end{figure}

%Given that we consider dual graphs of triangulations, it makes sense to
%which has been unknown until now.
Given that our scenario is based on visiting (dual) vertices, it is natural
to consider the worst-case behavior of LRV-v for the special class
of planar graphs of maximum degree 3 that can arise as duals of triangulations.
Until now, this has been an open problem. Moreover, it also makes sense to
consider the worst-case behavior of LRV-e for the same special graph class,
which is not covered by the work of Cooper et al.~\cite{cik+-drwug-11}.

\begin{theorem}
\label{th:lb.LRV-v}
There are dual graphs of triangulations (in particular, planar graphs with $n$ vertices of maximum degree 3), in which LRV-v
leads to a largest refresh time for a node that is exponential in $n$.
\end{theorem}

\begin{proof}
Consider the graph $G_D$ in Fig.~\ref{triang-LRV} colored in blue, which contains $\Theta(n)$ identical components connected in a chain.
We prove the claimed exponential time bound by recursively calculating the time taken to complete one cycle in the transition
diagram shown in Fig.~\ref{LRV}.

\begin{figure}[t]\centering
\includegraphics[width=\linewidth]{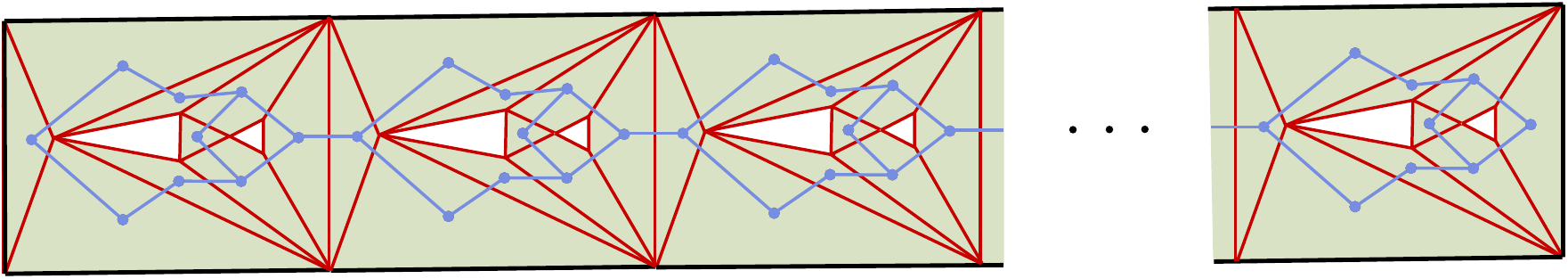}
\caption{A polygonal region with holes triangulated (red lines) with the dual graph $G_D$ of the triangulation (blue lines).} \label{triang-LRV}
 \end{figure}

\begin{figure}[h]\centering
%\begin{tabular}{cc}
 \includegraphics[width=0.4\textwidth]{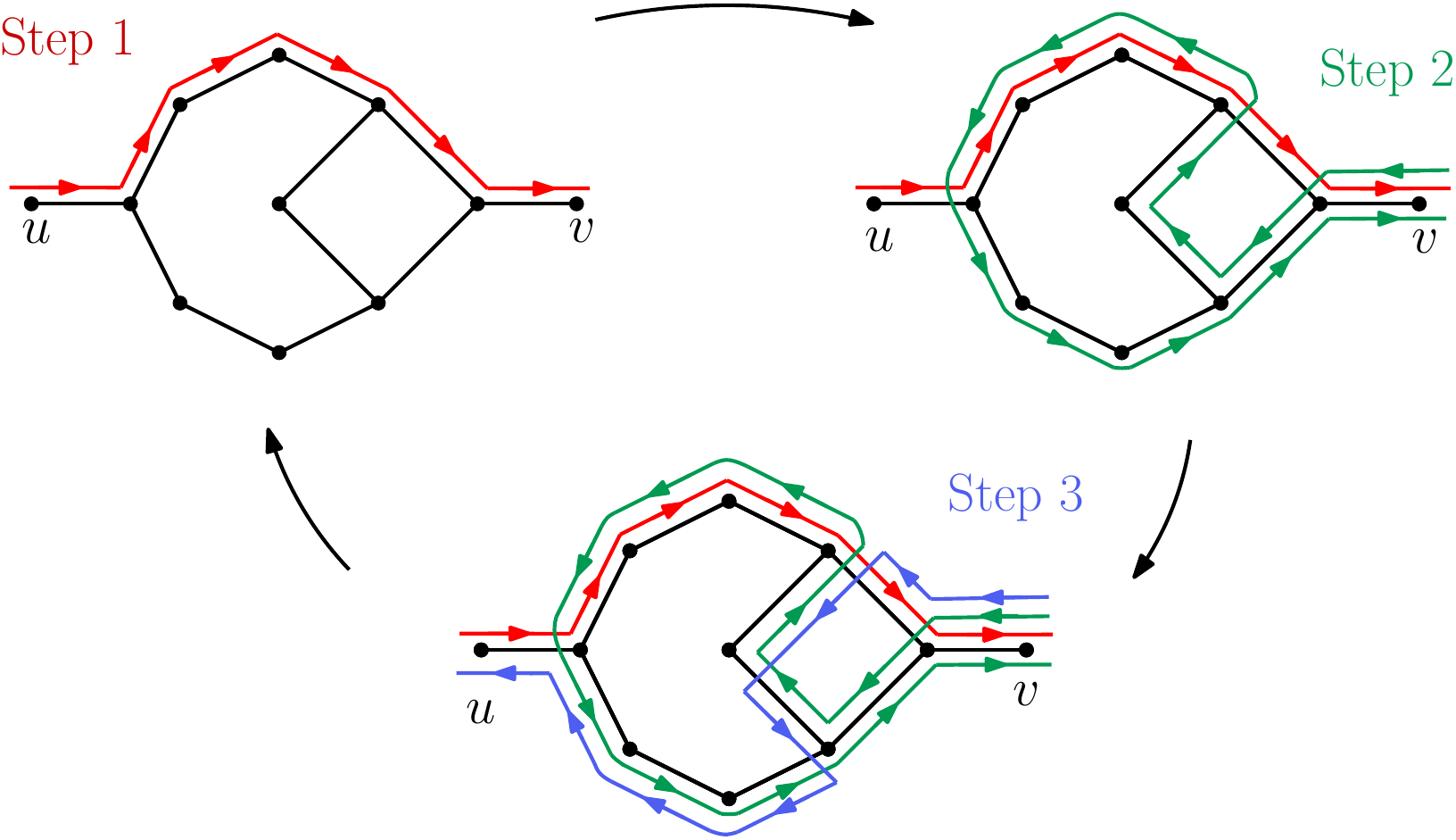}
\vspace*{2mm}
%&
 \includegraphics[width=0.4\textwidth]{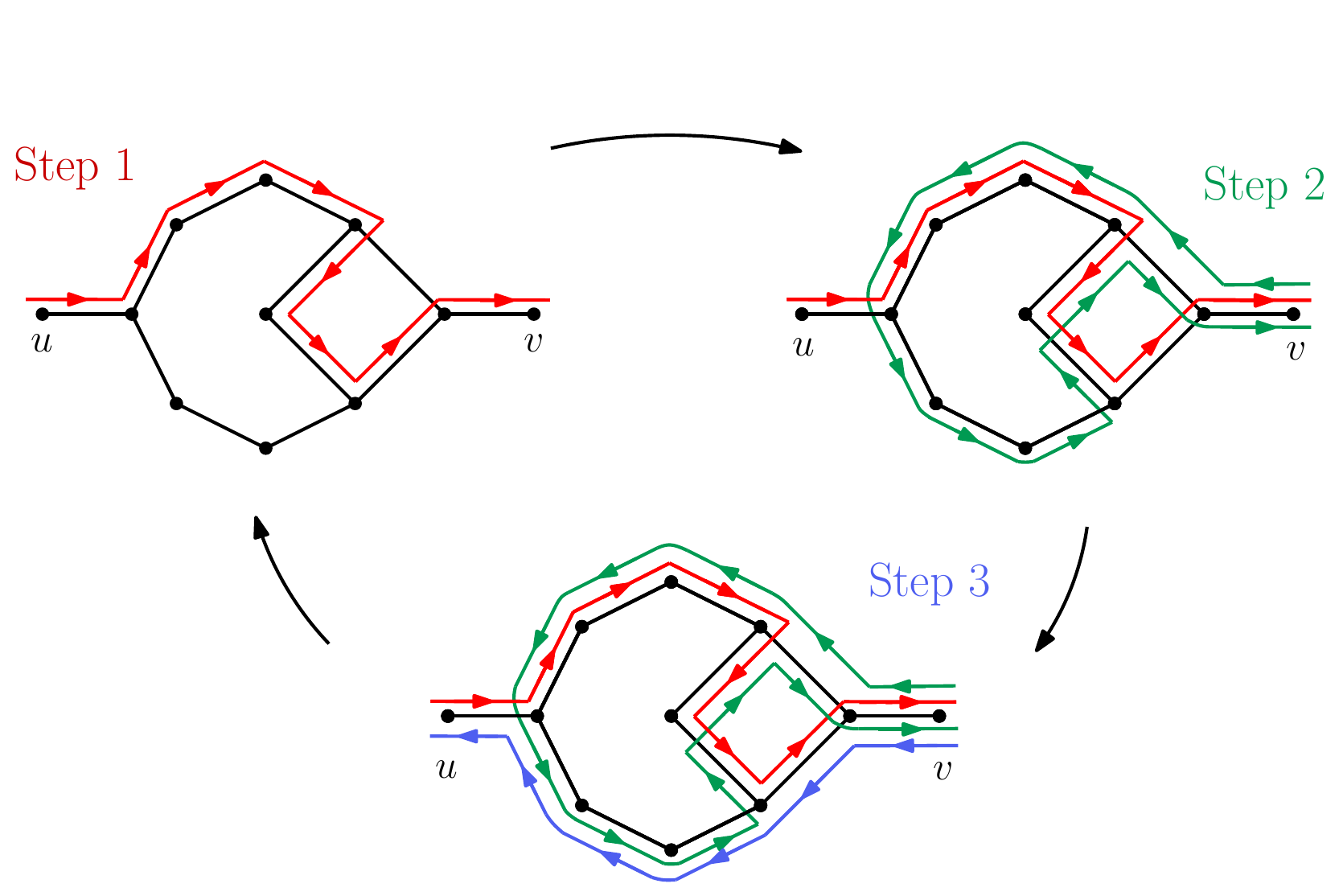}
%\\
%(a)
%&
%(b)
%\end{tabular}
\caption{Two possible alternating paths for the LRV-v strategy on each component of the graph $G_D$.} \label{LRV}
 \end{figure}

We monitor the movement of a robot from this situation onwards.
Let a robot take at least one unit of time to traverse an edge. Moreover, let $T_n$ denote the time taken to
complete one cycle of $G_D$, i.e., the time taken by a robot to start from and return to the first vertex of the first component of $G_D$.
From the possible paths illustrated in Fig.~\ref{LRV}, we can observe that the vertex $u$ is visited only during the beginning and
end of the cycle, while the vertex $v$ is visited twice in this cycle. It is not hard to check that the summation of visits to all edges
in one component during one cycle is 26. Using this we can see a simple recursion as follows:
$$T_n \geq 26 + 2 \cdot T_{n-1},\;T_0 = 0$$
Solving this equation, we get
$$T_n \geq 26 \cdot (2^n-1)$$
Utilizing the fact that vertex $u$ of the first component is visited only at the beginning and end of a cycle in $G_D$, we see that it is
visited after $T_n \geq 26 \ast (2^n - 1)$ units of time, which is exponential in the number $n$ of nodes of graph $G_D$, as claimed.
\end{proof}

As it turns out, the same negative result can be established for LRV-e, making both versions of LRV unsuitable for avoiding bad worst-case
behavior.

\begin{theorem}
\label{th:lb.LRV-e}
There are dual graphs of triangulations (in particular, planar graphs with $n$ vertices of maximum degree 3), in which LRV-e
leads to a largest refresh time for a node that is quadratic in $n$.
\end{theorem}

\begin{proof}
This proof follows the exact arguments as the proof of theorem \ref{LFVe-lb}
for the same graph $G_D$ represented in Fig. \ref{triang}.
As illustrated in Fig. \ref{pattern-LRVe}, each of $G_D$'s components is traversed
initially following the colored oriented paths from step 1 and further
alternating the paths from step $2k$ and $2k+1$.
 \begin{figure}\centering
 \includegraphics[width=0.37\textwidth]{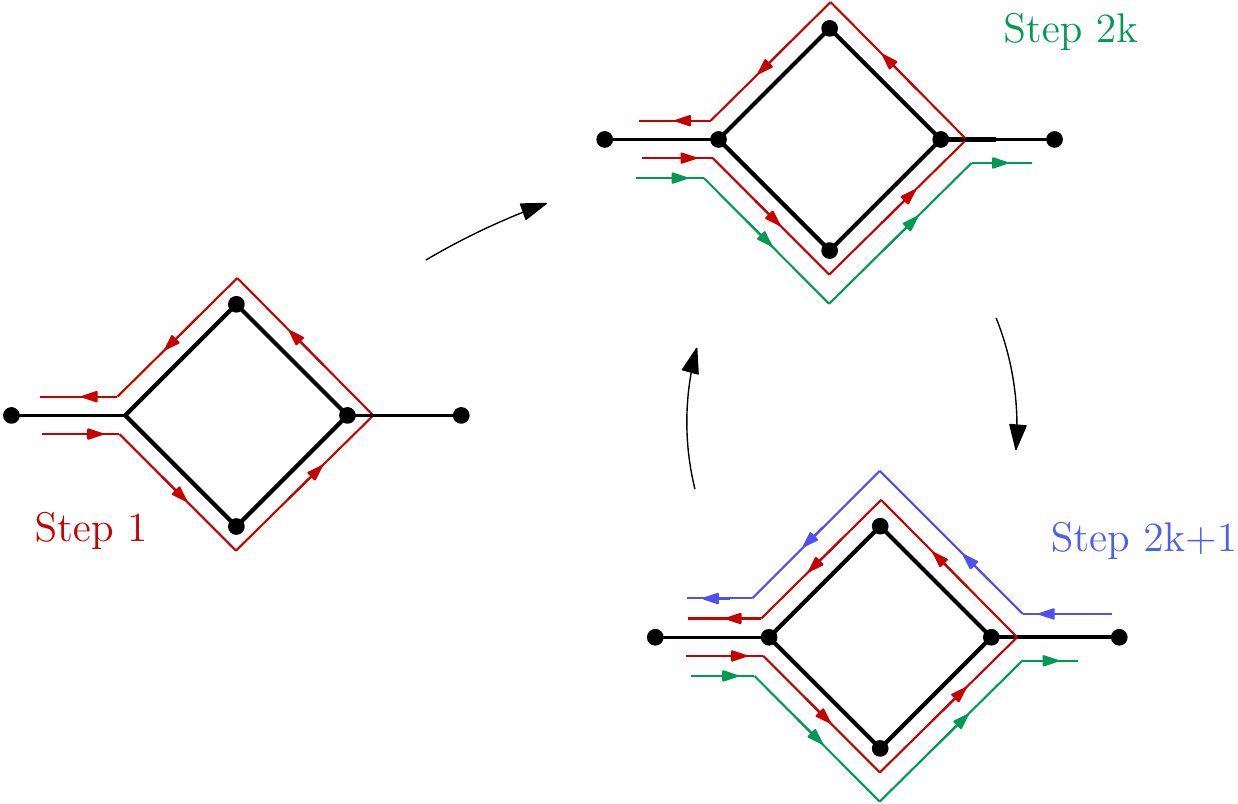}
\caption{LRV-e strategy on each component of the graph $G_D$.} \label{pattern-LRVe}
 \end{figure}
%may be removed
In other words, the first time a component is traversed, the path
changes direction and goes back to the start. The rest of the times
when the component is traversed, the direction does not change.
Thus, in order to traverse the $i$th component in the chain, we need
to traverse the first $i-1$ components in the chain.
\end{proof}

\subsection{Worst-Case behavior of LFV-v and LFV-e}

%\todo{****************************************}\\
%\todo{new work from Daniela and Alex has been inserted. I have rearranged it. We %probably have to cut some material
%for space reasons - maybe Lemma 3.6, Theorem 3.7, Theorem 3.8. Two figures have %not been checked in.
%Please check all of this!}
%\todo{****************************************}

In the following, we provide evidence that a small polynomial upper bound on the worst-case latency
is unlikely for LFV-v.
We start by showing some interesting properties of graphs explored under LFV-v.
It would seem at first that paths in the graph must be in nonincreasing frequency.
This is so as we always select the neighbor of lowest frequency. However, if
all  node's neighbors have the same or higher frequency, then
the destination node will have strictly larger frequency than the present
neighbor (see Figs. \ref{path} and \ref{staircase}).

\begin{figure}\centering
 \includegraphics[width=0.22\textwidth]{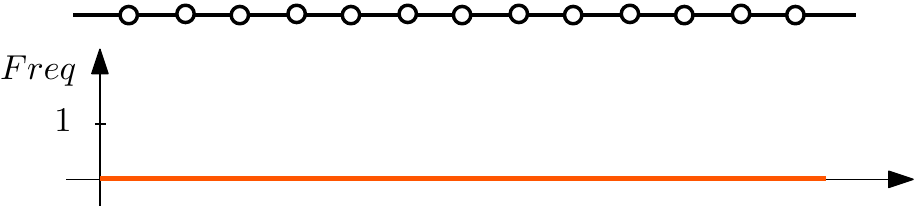}\hspace{0.5cm}
 \includegraphics[width=0.22\textwidth]{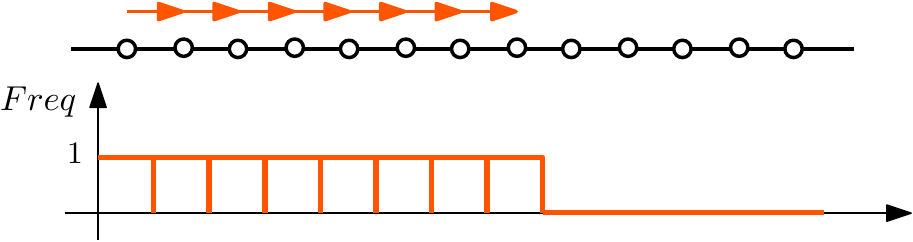}
%  \vspace{1cm}
%  \includegraphics[width=0.25\textwidth]{freq3.pdf}\\
\caption{A path being traversed from left to right with its frequency histogram below.
Initially all nodes have frequency zero. Then half way through the path traversal
nodes to left have frequency 1 and nodes to the right are still at zero.%
% Lastly when the path has been fully traversed all nodes have frequency 1.
} \label{path}
 \end{figure}

%Observe that we have \emph{two possible definitions for LRV-n} here. When all
%the neighbors of a node have higher frequency than the present node then it
%can either (1) visit the lowest frequency of the neighbors or (2a) it can remain in
%the present node for as many time units as necessary until its frequency
%matches that of the lowest neighbor or (2b) remain in the present node until
%its frequency exceeds by one that of the lowest neighbor.
%
%In either case this shows that indeed is possible to achieve a path that
%is increasing in frequency at each step as shown in Figure \ref{staircase}.

\begin{figure}\centering
 \includegraphics[width=0.21\textwidth]{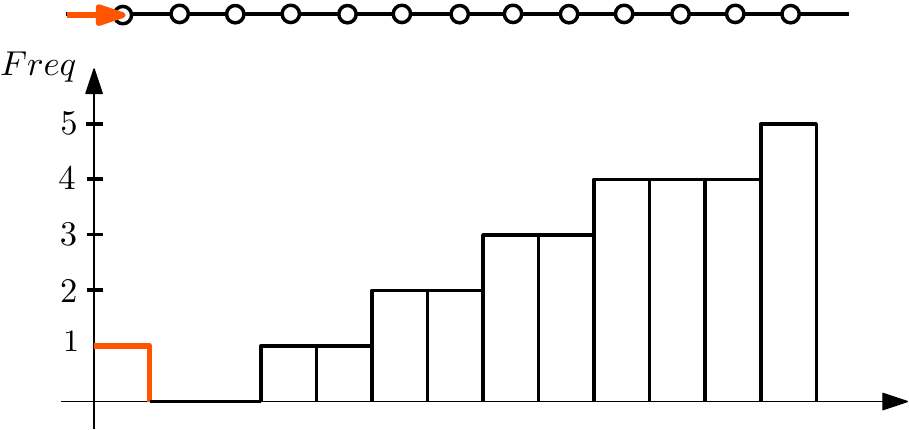} \hspace{0.5cm} \includegraphics[width=0.21\textwidth]{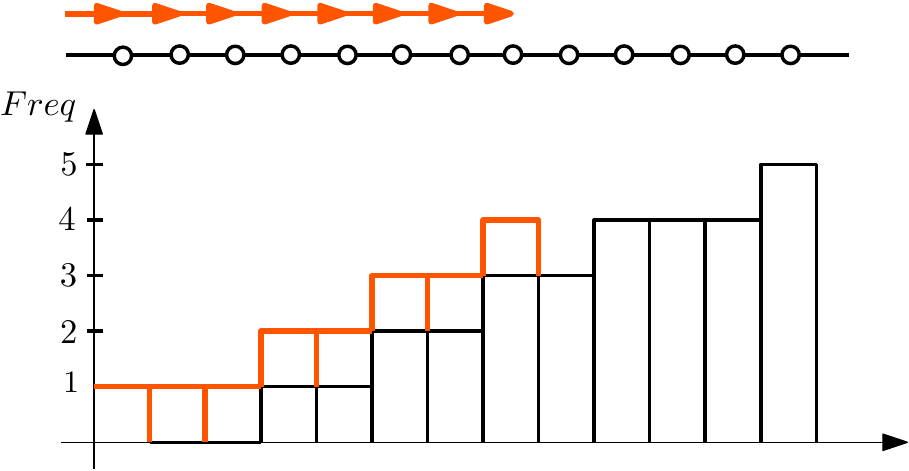}
%  \vspace{1cm}
%  \includegraphics[width=0.2\textwidth]{freq6.pdf}\\
\caption{A path with a corresponding staircase pattern in the histogram.} \label{staircase}
 \end{figure}

Lastly we observe that it is possible to create dams or barriers by having a
flower configuration  in the path (see Fig. \ref{petals}). We reach
the center of the flower and then take the loops or petals, thus increasing the
count of the center (see histogram on Fig. \ref{petals}).
Then the robot moves past the center node of the flower, which forms
a barrier that impedes the robot from traversing from right to left past the
center of the flower, until the count of the nodes to the right of the path has
risen to match that of the barrier.
\begin{figure}\centering
 \includegraphics[width=0.45\textwidth]{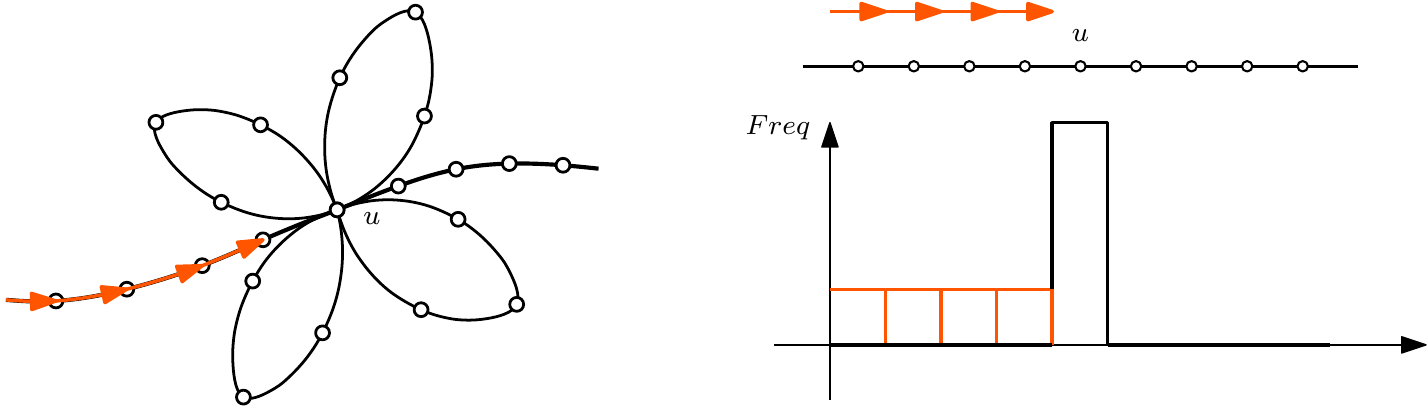}
\caption{A path with a ``flower'' configuration which creates a barrier.%
%The robot enters from the left, traverses the petals increasing the count
%of the center and exits to the right. The robot will not traverse past
%the center node until the frequency of the nodes to the right of it has
%risen to match the height of the barrier.
} \label{petals}
 \end{figure}
With these three basic configurations in hand, we can combine them to create a
graph in which the starting node $s$ has $\delta(s)$ neighbors, $\delta-1$ of which lead via
staircases with barriers to the last neighbor of $s$ which revisits
$s$ (see Fig. \ref{freq}).
That is for each time we go from $s$ to one of the first $\delta-1$ neighbors
we then climb a staircase up to the last common neighbor of $s$. Then from
that neighbor we enter each staircase from the ``high'' side until stopped
by a barrier, which makes us return to the neighbor of $s$, eventually
revisiting $s$ from this last neighbor. This shows the following theorem.

\begin{theorem}
\label{th:ratio}
There exists a configuration for LFV-v in which some neighbors of the starting
vertex have a frequency count of $k$, while the starting point has
a frequency count of $k\,\delta$. Moreover, the value of $k$ can be as high
as $\Theta(n/\delta)$.
\end{theorem}
\begin{figure}\centering
 \includegraphics[width=0.26\textwidth]{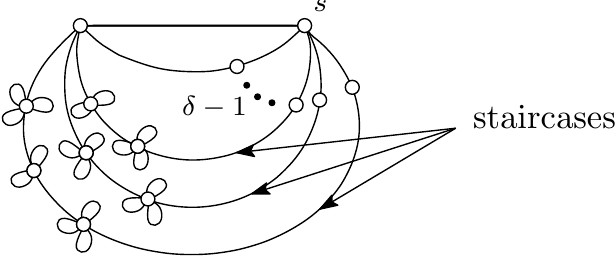}
\caption{A configuration in which the frequency of the starting point $s$ is
much larger than the majority of its neighbors.  } \label{freq}
 \end{figure}
This result provides some indication that the worst-case ratio between
smallest and largest frequencey labels of vertices may be exponential,
which would arise if we could construct an example in which $\Omega(\delta)$ linear
ratios as in Theorem~\ref{th:ratio} between neighbors occur.
From this it can be shown that at most $\delta^d$ steps are required to explore the graph,
where $d$ and $\delta$ are the diameter and the maximum degree of the graph.
\begin{theorem}
The highest frequency node in a graph with unvisited nodes has frequency
bounded by $\delta^d$.
\end{theorem}
\begin{proof}
Consider a path from an unvisited node to the node with highest frequency.
The path is of length at most the diameter $d$ of the graph. In each step
the increase in frequency is at most a factor $\delta$ over the unvisited node
hence the frequency of the most visited node is bounded by $\delta^d$.
\end{proof}
However, there is no known example of a dual of a triangulation graph
displaying this worst-case behavior.
\begin{theorem}\label{LFVe-lb}
There exist graphs with $n$ vertices of maximum degree 3, in which the largest
refresh frequency for a node is $\Theta(n^2).$
\end{theorem}
\begin{proof}
Consider the graph $G_D$ of Fig. \ref{triang} as described in its caption
which consists of a chain
of $\Theta(n)$ cycles of length 4 connected in series. As illustrated in Fig. \ref{LFV}, each component is traversed
initially following the colored oriented paths from step 1 and further
alternating the paths from step $2k$ and $2k+1$.
%, we require a total $\Theta(n)$ traversals, starting from the left
%for at least one of them to reach the rightmost point in this series.
\begin{figure}\centering
 \includegraphics[width=0.39\textwidth]{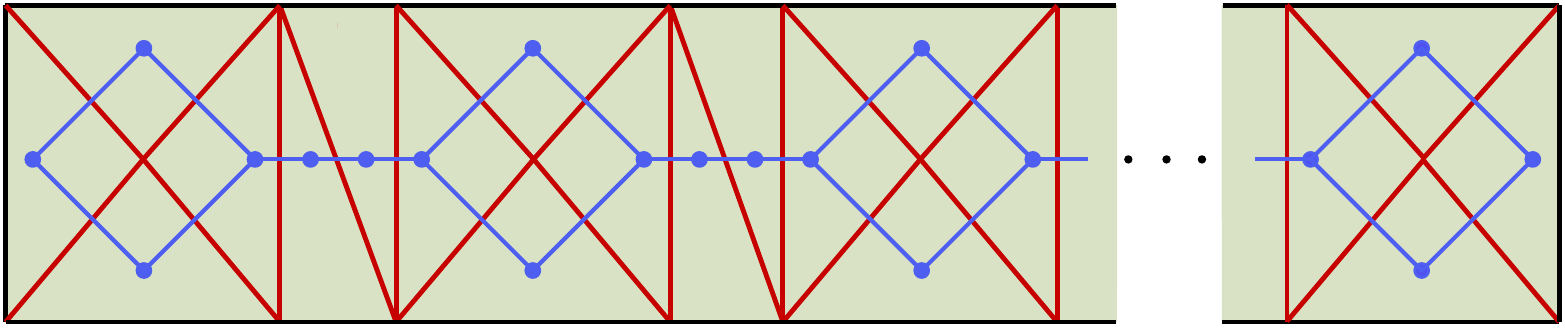}
\caption{This figure depicts (1) a rectangular polygonal region in black lines (2) its triangulation $G_P$ in red lines
and (3) the dual graph of the triangulation $G_D$ shown in blue lines.} \label{triang}
 \end{figure}
%More formally, we describe the worst-case patrolling of the graph showing
%the required quadratic time with respect to $n$.
%
 \begin{figure}[h]\centering
 \includegraphics[width=0.33\textwidth]{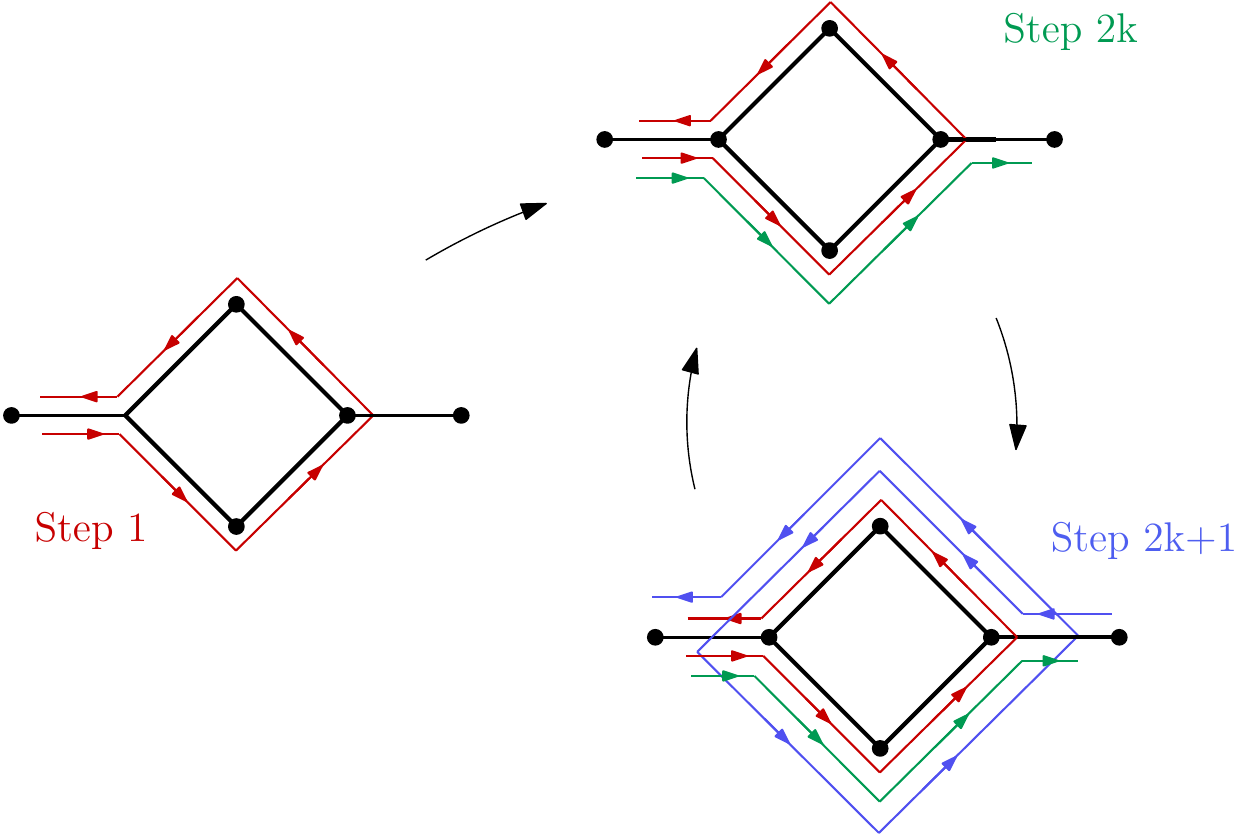}
\caption{LFV-v strategy on each component of the graph $G_D$.} \label{LFV}
 \end{figure}
In other words, the first time a component is traversed, the path
changes direction and goes back to the start. The rest of the times
when the component is traversed, the direction does not change. Thus, in order to traverse the $i$th component in the chain, we need
to revisit the first $i-1$ components in the chain.
\end{proof}

Thus, LFV-v may also display bad worst-case behavior.
Fortunately, the following was proven by Cooper et al.~\cite{cik+-drwug-11}
for the worst-case behavior of LFV-e.

\begin{theorem}
\label{th:ub.LFV-e}
In a graph $G$ with at most $m$ edges and diameter $d$, the latency of each edge when carrying out
LFV-e is at most $O(m\cdot d)$.
\end{theorem}
This allows us to establish a good upper bound on LFV-e in our setting.
\begin{corollary}
\label{co:ub.LFV-e}
Let $G_D=(V_D,E_D)$ be the dual graph of a triangulation, with $|V_D|=n$ vertices and diameter $d$. Then
the latency of each vertex when carrying out LFV-e is at most $O(n\cdot d)$.
\end{corollary}

\begin{proof}
Since $G_D$ is planar, it follows that $n\in\Theta(m)$, where $m=|E_D|$ is the number of edges
of $G_D$. Because patrolling an edge requires visiting both of its vertices, the claim follows from the
upper bound of Theorem~\ref{th:ub.LFV-e}.
\end{proof}

We note that this bound can be tightened for regions with small aspect ratio, for which
the diameter is bounded by the square root of the area.

\begin{corollary}
\label{co:aspect}
For regions with diameter $d\in O(\sqrt{n})$, the
latency of each dual vertex when carrying out LFV-e is at most $O(n^{1.5})$.
\end{corollary}

\begin{figure*}[t!]
\renewcommand{\figDim}{1.2in}
\renewcommand{\figDimT}{1.2in}
\centering
\subfloat[][Simulation environment]{
\label{fig:MaxRefreshSimulationSetup}
\includegraphics[height=\figDim]{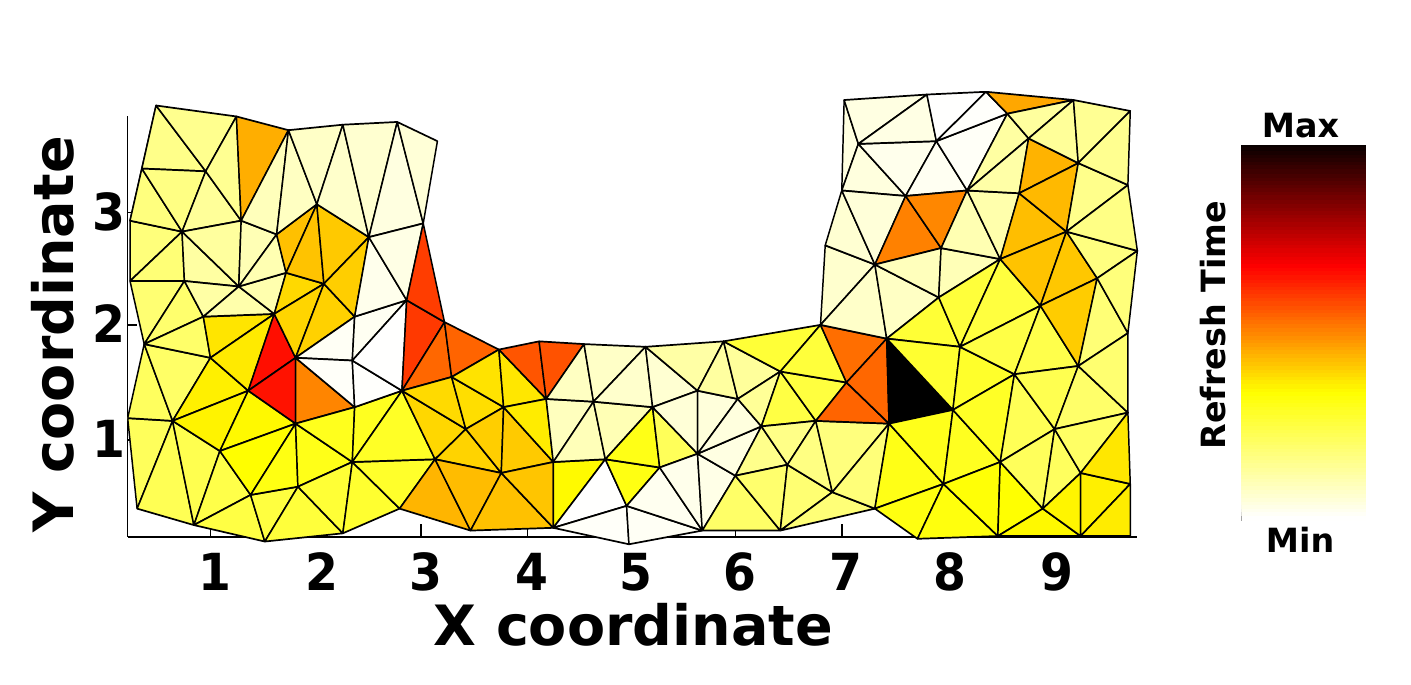}}
\subfloat[][Patrol trace for one robot]{
\label{fig:TracePatrol}
\includegraphics[height=\figDimT]{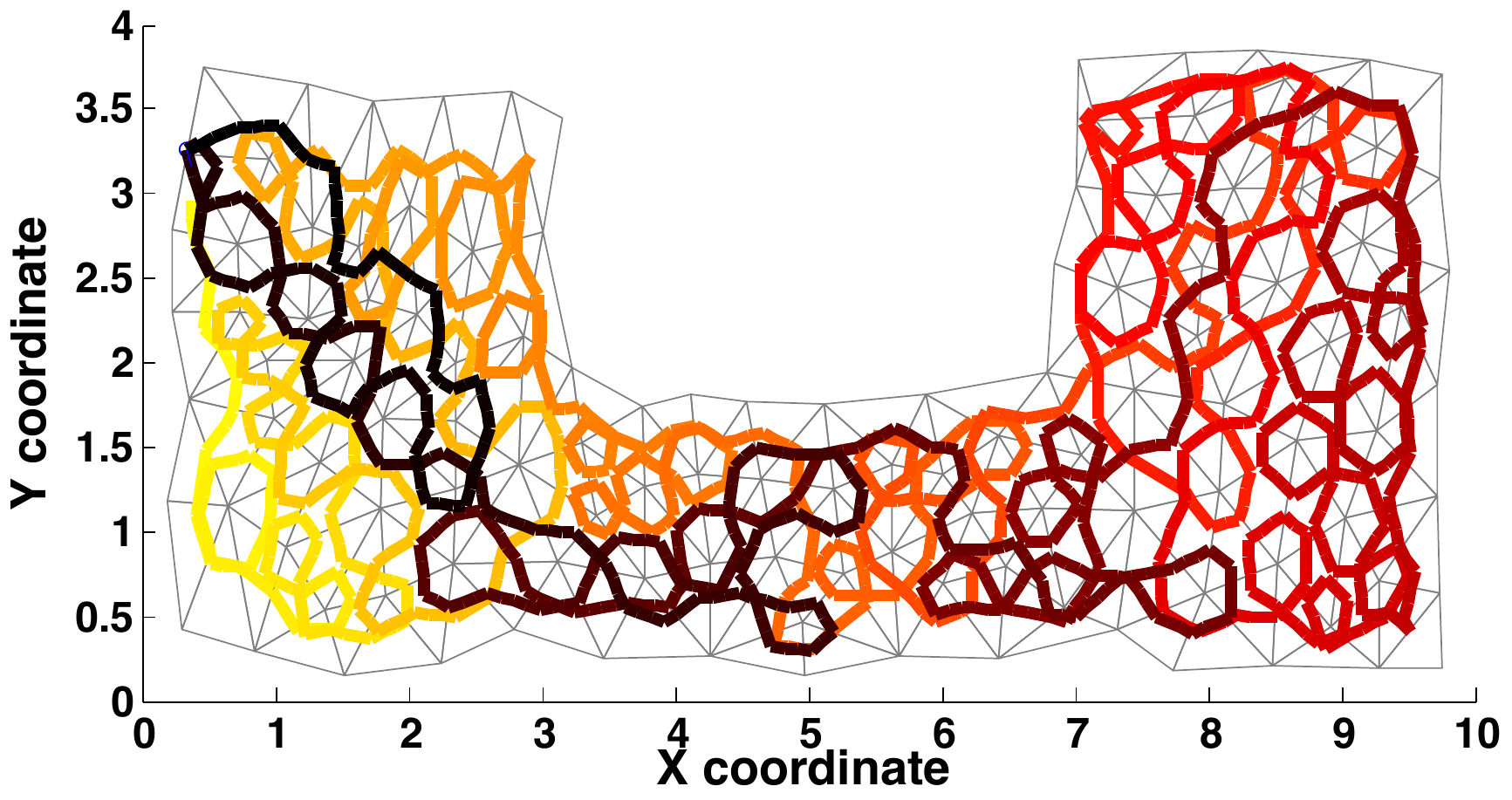}}\\
\subfloat[][Sample experimental trial]{
\label{fig:MaxRefreshSimulationDataExample}
\includegraphics[height=\figDim]{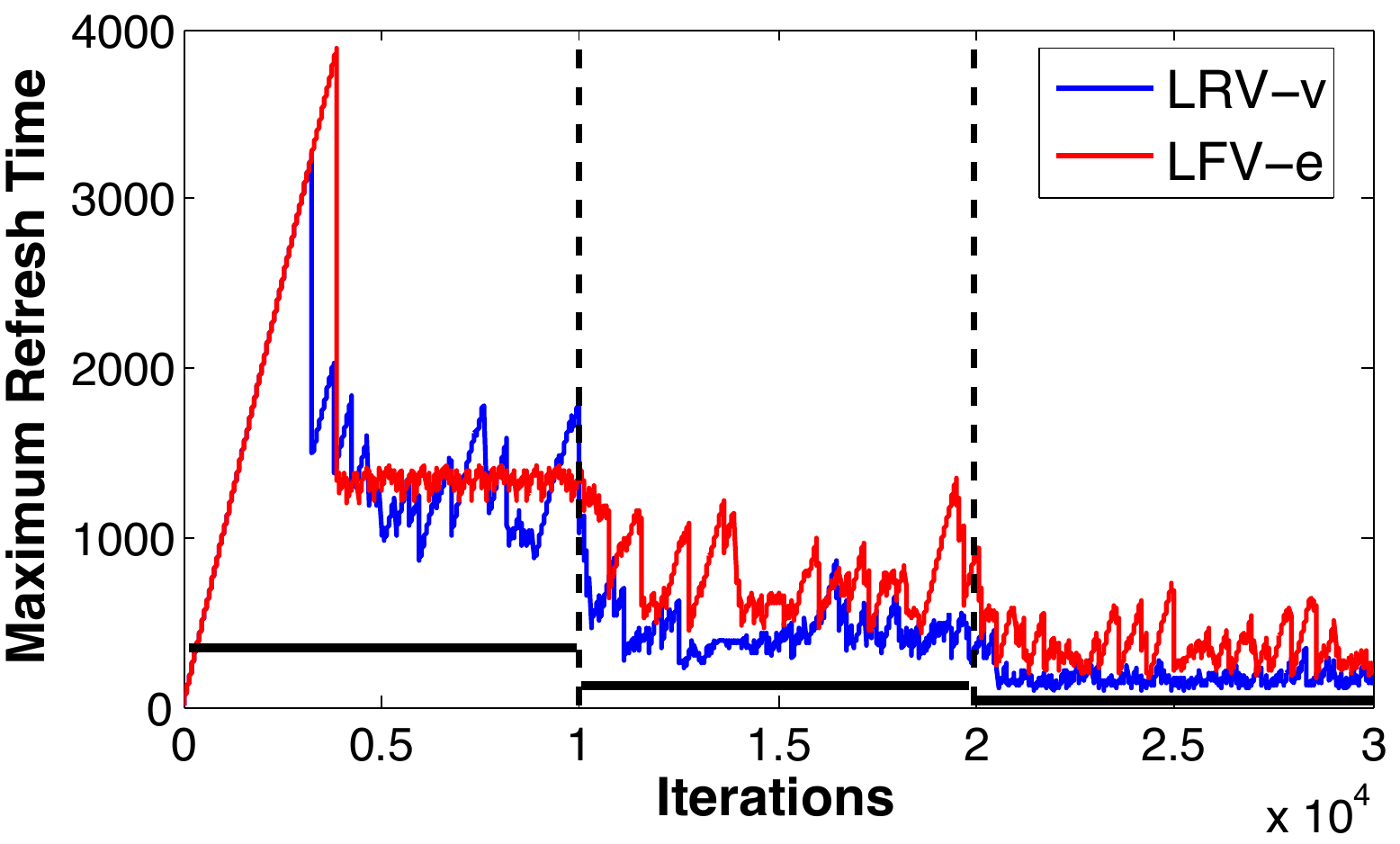}}
\subfloat[][Average max refresh time data]{
\label{fig:MaxRefreshSimulationData}
\includegraphics[height=\figDimT]{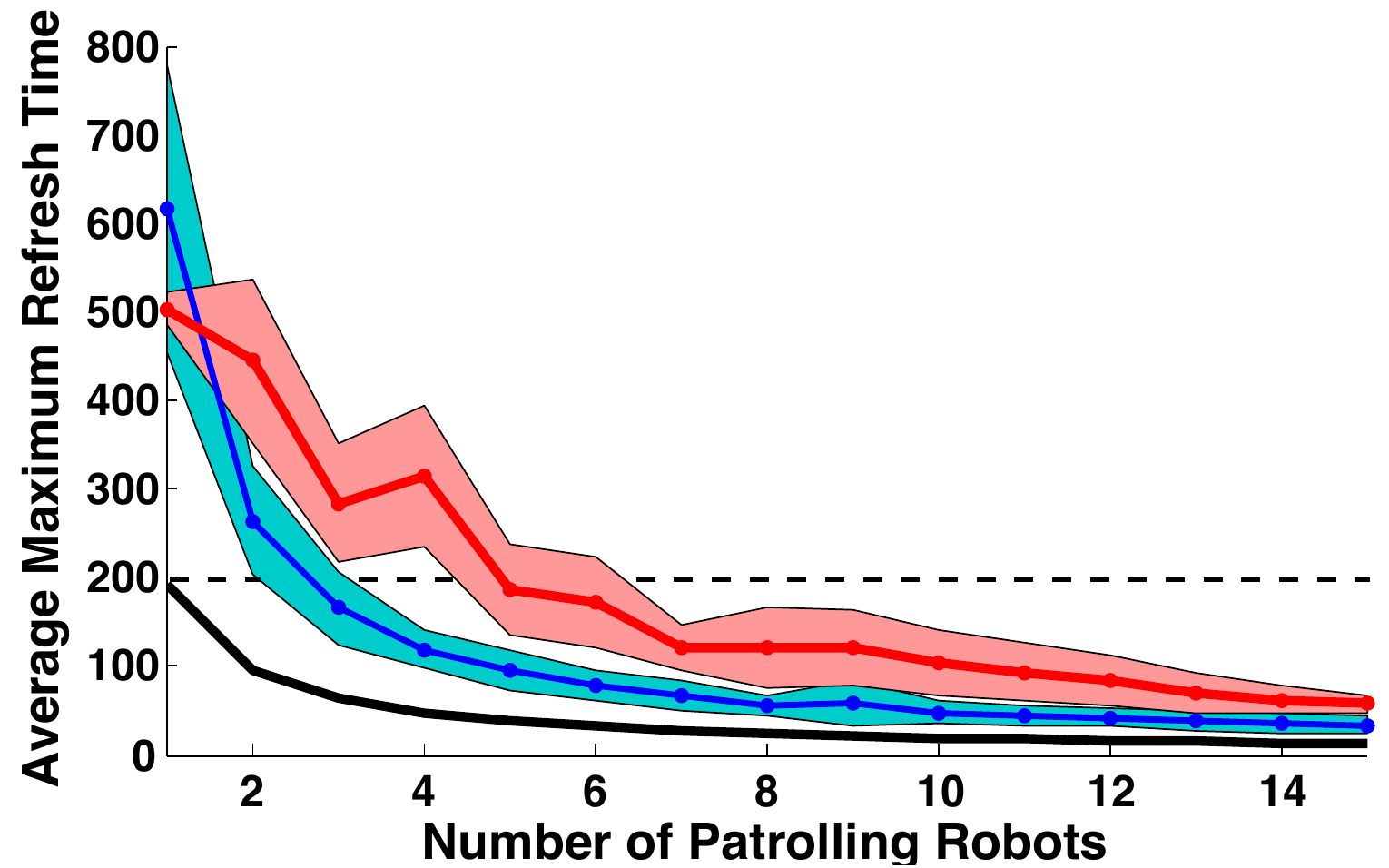}}
\subfloat[][Robot experiment]{
\label{fig:MaxRefreshRobotData}
\includegraphics[height=\figDimT]{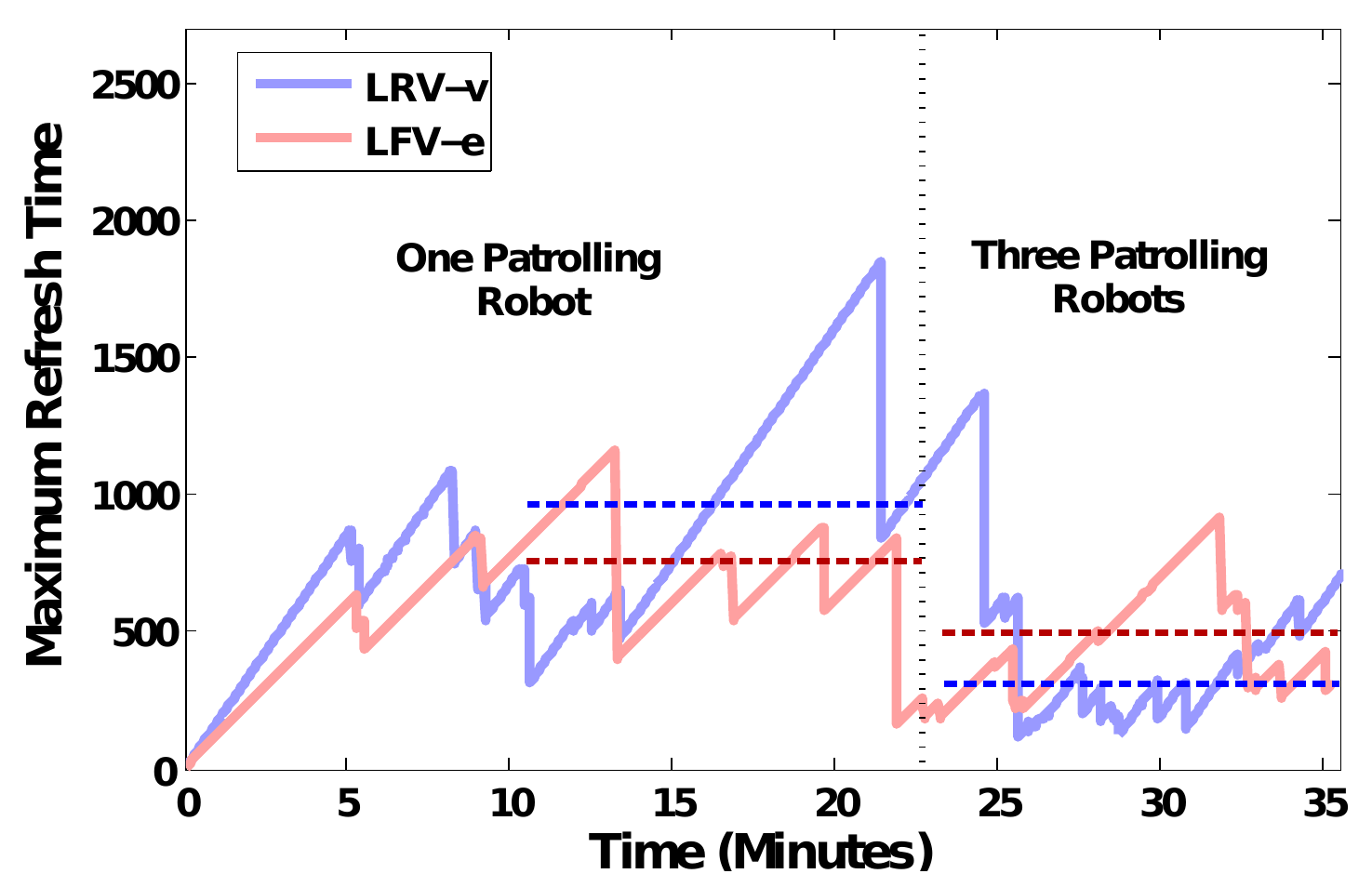}}
\caption{
\textbf{\protect\subref{fig:MaxRefreshSimulationSetup}} Our simulated environment contains 193 triangles. The color of each triangle indicates the refresh time of the triangle. The black triangle in the middle of the triangulation has the maximum refresh time.
\textbf{\protect\subref{fig:TracePatrol}} An example cyclic path from a single patrolling robot. The color of each path segment indicates the age, darkest color is the most recent path segment.  This path produces a worst-case coverage with the tight bound on variance, as can be seen in the red trace of Fig.~\protect\subref{fig:MaxRefreshSimulationData}.
\textbf{\protect\subref{fig:MaxRefreshSimulationDataExample}}  Maximum Refresh Time using LRV-v and LFV-e. Simulations start with one robot, and we put two and six additional robots at 10000 and 20000 iterations. Solid black lines indicate the lower bound of maximum refresh time of each patrolling robot, called baseline.
\textbf{\protect\subref{fig:MaxRefreshSimulationData}} The trend of path length of each patrolling robot according to the number of patrolling robots. The black line shows the lower bound for a perfect set of disjoint patrolling cycles.  The dotted line shows the best performance of a single robot.  We ran the simulations using one to fifteen patrolling robots, $8\%$ of the number of triangles.
\textbf{\protect\subref{fig:MaxRefreshRobotData}} Data from our patrolling experiment.  Our experimental setup is shown in Fig~\ref{fig:NavPatrol}.  It consisted of 16 triangulation robots, and one or three patrolling robots.  The experimental data is very similar to our simulation results.
\label{fig:expData}
}
\end{figure*}

\section{Experimental Results}

In order to validate and test our results, we performed a number of simulations, as
well as experiments with a real-world platform. Fig.~\ref{fig:rone} shows the r-one platform,
developed at Rice University.

\begin{figure}[h]
\begin{center}
\includegraphics[width=.45\linewidth]{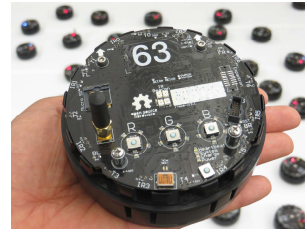}
\caption{Snapshot of the r-one robot used in the experiments.
%\todo{show dual graph in blue, increase line weights in primal graph (and robots), replace hop counts with timestamps}
}
\label{fig:rone}
\end{center}
\end{figure}

An example simulation tun is shown in Fig~\ref{fig:MaxRefreshSimulationSetup}.
Fig.~\ref{fig:MaxRefreshSimulationDataExample} gives results from a real-world patrolling experiment
with $1 - 3$ robots.

With one robot in the simulation, the LFV-e policy produces a closed path that resembles a Hamiltonian path,
shown in Fig.~\ref{fig:TracePatrol}; as such a path is a theoretical
lower bound, which is achievable only in exceptional cases, it is intuitively
clear that we are getting excellent results.
As we add more patrolling robots, the consolidated
Hamiltonian-like path is perturbed and the variance of the maximum refresh time
increases, shown in the red line (after 10K iterations) in
Fig.~\ref{fig:MaxRefreshSimulationDataExample}. Analyzing the paths of multiple
robots show that there is some overall rebalancing in addition to finding new
individual subtours. This corresponds to handling the two aspects of balancing loads
and finding tours, which are both computationally hard in an offline setting.
It can be seen that the LFV-e policy carries out more delicate operations, leading
to a greate amount of initial perturbation. However, it is clear from the aggregated
results shown in Fig.~\ref{fig:MaxRefreshSimulationDataExample} that eventually, LFV-e yields superior results.
Shown is the average maximum $RT_{\Delta_i}(t)$ for larger numbers of
patrolling robots.  The solid black line shows the lower bound for an optimal
set of patrols.  Assuming all of the patrolling robots are on their own
Hamiltonian path, we can compute this lower bound by $\frac{|H(G)|}{r}$, where
$|H(G)|$ is the length of a Hamiltonian cycle of the dual graph $G$ (the number
of triangles), and $r$ is the number of patrolling robots.  Despite of these differences
between LRV-v and LFV-e, it is remarkable that both policies
performed well, with a clearly evident linear speedup according to increasing number $r$
of robots. (Refer to the theoretically optimal lower bound mapped by the black hyperbola.)

Our hardware experimental setup is shown in Fig~\ref{fig:NavPatrol}, and the
data for a experiment is shown in Fig.~\ref{fig:MaxRefreshRobotData}.  The
experiment started with one patrolling robot in the blue triangle.  At 22
minutes into the experiment, we added two more robots, one in the blue
triangle, and the other in the triangle to its left.  The horizontal lines show
the average refresh time in the last 10 minutes of patrolling.  The LRV-v
averages are 984 for one robot and 365 for three, and the LFV-e averages are
750 for one robot and 485 for three.  Units are robot rounds.  These data match
our simulation results nicely.
The experiment starts with one navigating robot; we add
two more robots at 10k iterations, and then another 5 at 20k iterations.  As we
deploy more navigating robots, the maximum $RT_{\Delta_i}(t)$ decreases, seen
in the red line. Note that here, too, LFV-e has the longest initialization before it
discovers all of the triangles, the variance in the max refresh time is the
smallest, once the patrolling routine has settled down.

%\begin{figure}[h]
%\begin{center}
%\includegraphics[width=\linewidth]{trace3.pdf}
%\caption{An example cyclic path from a single patrolling robot. The color of each path segment indicates the age, darkest color is the most recent path segment.  This path produces a worst-case coverage witha very tight bound on variance, as can be seen in the red trace of Fig.~\ref{fig:MaxRefreshSimulationData}.
%}
%\label{fig:TracePatrol}
%\end{center}
%\end{figure}

\section{Conclusion}
\label{sec:Conclusion}

We have demonstrated how a combination of a weak stationary swarm with a set of mobile robots with limited capabilities
can perform well for purposes of patrolling and surveillance of a region. The simple policy LRV (based on purely local information without
any sophisticated communication between devices) allows complete coverage of the mapped area, but may lead to exponential refresh times
for all portions in specific worst-case examples; the alternative policy LFV is of similar simplicity can protect against
this worst-case behavior, provided it is applied to crossed edges of triangular subregions instead of the subregions themselves.
In realistic simulations as well as real-world experiments, both policies perform quite well. Most remarkably, they display
linear speedup for the visiting frequencies of the surveyed subregions, implying that even simple local policies on
weak robots can vastly outperform single, powerful robots with full information and strong computational capabilities.

%\addtolength{\textheight}{-6cm}   % This command serves to balance the column lengths

%%%%%%%%%%%%%%%%%%%%%%%%%%%%%%%%%%%%%%%%%%%%%%%%%%%%%%%%%%%%%%%%%%%%%%%%%%%%%%%%

%%%%%%%%%%%%%%%%%%%%%%%%%%%%%%%%%%%%%%%%%%%%%%%%%%%%%%%%%%%%%%%%%%%%%%%%%%%%%%%%

%%%%%%%%%%%%%%%%%%%%%%%%%%%%%%%%%%%%%%%%%%%%%%%%%%%%%%%%%%%%%%%%%%%%%%%%%%%%%%%%
%\section*{APPENDIX}
%\section*{ACKNOWLEDGMENT}
%
%The preferred spelling of the word ÒacknowledgmentÓ in America is without an ÒeÓ after the ÒgÓ. Avoid the stilted expression, ÒOne of us (R. B. G.) thanks . . .Ó  Instead, try ÒR. B. G. thanksÓ. Put sponsor acknowledgments in the unnumbered footnote on the first page.

%%%%%%%%%%%%%%%%%%%%%%%%%%%%%%%%%%%%%%%%%%%%%%%%%%%%%%%%%%%%%%%%%%%%%%%%%%%%%%%%

\bibliographystyle{plain}
\bibliography{lit,bibliography,mclurkin-bibliography}

\end{document}